\documentclass[letterpaper, 10 pt, conference]{ieeeconf}  

\IEEEoverridecommandlockouts                              

\usepackage{graphicx} 
\usepackage{amsmath} 
\usepackage{amsfonts}
\usepackage{amssymb}

\usepackage{algorithm}
\usepackage{algpseudocode}
\usepackage{subfigure}
\usepackage{caption}
\usepackage{breqn}

\newtheorem{definition}{\bf Definition}
\newtheorem{lemma}{\bf Lemma}
\newtheorem{proposition}{\bf Proposition}
\DeclareMathOperator*{\argmin}{arg\,min}

\title{\LARGE \bf
Towards a Framework for Tracking Multiple Targets: Hybrid Systems meets Computational Geometry 
}

\author{Guillermo J. Laguna, Rui Zou and Sourabh Bhattacharya
\thanks{Guillermo J. Laguna, Rui Zou and Sourabh Bhattacharya are with the Department of Mechanical Engineering,
        Iowa State University, Ames, Iowa 50011, USA
        {\tt\small \{gjlaguna, rzou, sbhattac\}@iastate.edu}}%
}

\begin{document}

\maketitle
\thispagestyle{empty}
\pagestyle{empty}

\begin{abstract}

We investigate a variation of the art gallery problem in which a team of mobile guards tries to track an unpredictable intruder in a simply-connected polygonal environment. In this work, we use the deployment strategy for diagonal guards originally proposed in \cite{O'Rourke:1983}. The guards are confined to move along the diagonals of a polygon and the intruder can move freely within the environment. We define critical regions to generate event-triggered strategies for the guards. We design a hybrid automaton based on the critical regions to model the tracking problem. Based on reachability analysis, we provide necessary and sufficient conditions for tracking in terms of the maximal controlled invariant set of the hybrid system. We express these conditions in terms of the critical curves to find sufficient conditions for $\lfloor \frac{n}{4} \rfloor$ guards to track the mobile intruder using the reachability analysis.

\end{abstract}

\section{Introduction}

\label{sec:intro}

Target tracking refers to the problem of planning the path of a mobile observer that should keep a mobile target within its sensing range. It arises in numerous applications involving monitoring and surveillance \cite{Li:1999,Briggs:1996,Hsu:2008}. The general visibility-based target tracking problem is the following; a team of autonomous mobile agents equipped with vision sensors is deployed as observers to track a team of mobile targets in the environment. To achieve this, the agents need to coordinate among themselves to ensure successful tracking. In such scenarios, an important question that arises for a network designer is the following: What is the minimum number of robots that need to be deployed in order to successfully perform the tracking task? In this work, we leverage results from art-gallery problems \cite{O'Rourke:1983} to explore the aforementioned problem.  

The art gallery problem is a well-studied visibility problem in computational geometry. A simple solution to the tracking problem in bounded environments is to cover the polygon representing the environment with sufficient number of observers. However, it has been shown that the problem of computing the minimum number of guards required to cover a simply connected polygon is NP-hard \cite{Lee_complexity}. Since covering a polygon is a specific instance of the tracking problem (an instance in which the intruder is infinitely fast), the problem of finding the minimum number of guards to track an intruder is at least as hard as the problem of finding the minimum number of guards for covering the polygon. Therefore, we try to find a reasonable upper bound on the minimum number of guards required to track a mobile intruder. In \cite{O'Rourke:1987}, it is shown that $\lfloor\frac{n}{3}\rfloor$ static guards with omnidirectional field-of-view is sufficient to cover the entire polygon.  In the case of mobile guards, it is shown that $\lfloor\frac{n}{4}\rfloor$ guards are sufficient to ensure that every point inside the polygon is visible to at least one guard, and \cite{O'Rourke:1983} provides an algorithm to deploy them. It implies that if these guards are allowed to move at infinite speed, $\lfloor\frac{n}{4}\rfloor$ of them are sufficient to track the intruder. In this work, we propose strategies for bounded speed guards to maintain a line-of-sight with an intruder. 

Numerous techniques have been proposed to obtain the minimum number of guards required for the classical art gallery problem. In \cite{Ghosh2010}, an approximation algorithm is proposed to provide a solution to the minimum number of vertex and edge guard. This solution is within $O(\log n)$ times of the optimal one. Heuristic techniques based on greedy algorithms and polygon partition are introduced in \cite{Amit2010}. Researchers have also investigated variations of the classical art gallery problem for specific environments. In \cite{Hoffmann:1990}, the authors show that the sufficient condition of $\lfloor \frac{n}{4} \rfloor$ mobile guards are sufficient to cover a rectilinear polygon with holes.

There have been some efforts to deploy multiple observers to track multiple targets. \cite{parker2002distributed} presents a method of tracking several evaders with multiple pursuers in an uncluttered environment. In \cite{jung2002tracking}, the problem of tracking multiple targets is addressed using a network of communicating robots and stationary sensors. In \cite{zhang2016multi}, the authors propose the idea of {\it pursuit fields} for a team of observers to track multiple evaders in an environment containing obstacles. In \cite{zou2015visibility}, we addressed the problem of decentralized visibility-based target tracking for a team of mobile observers trying to track a team of mobile targets. The aforementioned works focus on investigating motion strategies for free guards to track multiple intruders in the environment. In contradistinction, this work deals with a scenario in which the paths of the guards are prespecified and the problem is to explore reactive strategies for the guards to construct trajectories that can track an unpredictable intruder.  

In this work, we use the theory of hybrid systems to provide our tracking guarantees. The area of hybrid systems is defined as the study of systems that involve the interaction of discrete event and continuous time dynamics, with the purpose of proving properties such as reachability and stability \cite{Tomlin:2000},\cite{Ding:2011}. Hybrid systems are characterized by continuous systems with a mode-based operation, where the different modes correspond to different continuous dynamics. One of the most commonly used hybrid system models is an hybrid automaton. It combines state-transition diagrams for discrete behavior with differential equations for continuous behavior. In \cite{Ding:2011}, a methodology for computing reachable sets for hybrid systems is presented. In this work we model the tracking problem as an hybrid automaton and perform a reachability analysis from it to obtain tracking guarantees.

Some fundamental concepts from the area of hybrid systems were introduced in \cite{Lygeros:2012}. There are many research directions regarding hybrid systems. In \cite{Mitchell:2002}, it is demonstrated that reachability algorithms using level set methods and based on the Hamilton-Jacobi PDE can be extended to hybrid systems whose dynamics are described by differential algebraic equations. For the problem of UAV traffic management problem, it is required to ensure that the safety requeriments of large platoons of UAVs flying simultaneously are met. In the past, hybrid systems have been used to model pursuit-evasion games. In \cite{Isler:2004}, a provable solution for visibility-based pursuit-evasion games in simply-connected environments is presented. It considers the optimal control solution for the differential pursuit-evasion game and also the discrete pursuit-evasion game on the graph representing the environment. In \cite{Li:2012}, a pursuit-evasion game in which UAVs must follow RF emitters is considered. Since the RF emitter can take both continuous and discrete actions. The properties of the game are studied as an hybrid system and optimal strategies were derived for both parties. In this work, the environment is partitioned in different regions, and for each one of them there a different set of guards moves to track the intruder. Each region is associated with a discrete state, and since the guards have continuous dynamics, there are also continuous states in the system. We perform a reachability analysis to determine if it is possible to reach states for which tracking is not guaranteed. The contributions of this work are as follows:

\begin{enumerate}
\item Given a triangulation of a polygon, we present a classification of the triangles based on the coverage provided by the guards, and a classification of the guards based on the type of triangles being covered by them.
\item We present an hybrid automaton that can be used to model the tracking problem addressed in this paper and present a reachability analysis to determine if the tracking of the intruder can be ensured.   
\item An upper bound on the ratio between the maximum speeds of the intruders and the guards to ensure persistent tracking is presented.
\item We present sufficient conditions for $\lfloor \frac{n}{4} \rfloor$ guards to track the intruders using the proposed strategy.
\end{enumerate}

This paper is organized as follows. In Section \ref{sec:problem}, we present the problem formulation. In Section \ref{sec:procedure}, we define different types of guards and regions that determine the reactive strategy of the guards. In Section \ref{sec:overall}, the tracking problem is modeled as an hybrid system, so we can use a reachability analysis in Section \ref{sec:reachability} to determine if the intruder can be tracked all the time. In Section \ref{sec:speed}, we obtain an upper bound for the speed ratio which ensures that the proposed strategy works. We conclude in Section \ref{sec:conclusion} with some future work.

\section{Problem Statement}
\label{sec:problem}

Given a simply connected polygonal environment $P$, we consider the triangulation of $P$ which is its partition into a set of disjoint triangles such that the vertices of the triangles are vertices of $P$. The diagonals of the triangulation are line segments that connect any two vertices of the triangulation. Let $G$ be the graph that represents the triangulated polygon such that $V(G)$ (vertex set of $G$) corresponds to the vertices of $P$, and $E(G)$ (edge set of $G$) corresponds to the diagonals of the triangulation of $P$.

We consider a team of mobile guards $S_g$ and an unpredictable intruder $I$ in the environment. Initially, $ \lfloor\frac{n}{4}\rfloor$ diagonal guards are placed based on the deployment proposed in \cite{O'Rourke:1983}. Let $g_i$ denote guard $i$ with $i \in \mathbb{N}$, confined to a diagonal $h_i$. Endpoints of $h_i$ are denoted by $v_j(i)$ with $j \in \{1,2\}$. The intruder and the guards can move or stay motionless. When the intruder moves it has a constant speeds $v_e$. When a guard moves, it has a constant speed $v_g$ (all guards have the same speed). Given the initial position of the intruder $p_I(0)$, and the speed ratio $r=\frac{v_e}{v_g}$, our objective is to find coordination and tracking strategies for the guards to ensure that the intruder is visible to at least one guard at all times. We present some important definitions that are used throughout the paper. $l_i$ is the length of $h_i$, and $d_M^i=l_i r$ is the ``Maximum'' distance that the intruder can travel while $g_i$ moves across $h_i$. We also classify the triangles in the triangulation of the environment. Refer to Figure \ref{fig:triangles}. In all figures, red segments illustrate the diagonals associated to guards, and the light gray region is the exterior of $P$.

\label{subsec:def}
\begin{definition}
\label{def:1} \hfill
\begin{enumerate}
\item Safe Triangle: A triangle is called \textit{safe} if it can be covered at any time the intruder enters it\footnote{We say that a triangle is ``covered'' when there is at least one guard at its boundary.}. They are shaded pale blue in Figure \ref{fig:triangles}.
\item Unsafe Triangle: A triangle is called \textit{unsafe} if there is only one guard that can cover it from one of its endpoints. They are shaded orange in Figure \ref{fig:triangles}.
\item Regular Triangle: A triangle is called \textit{regular} if it is neither safe nor unsafe. They are shaded white in Figure \ref{fig:triangles}.
\item Safe Zone: A \textit{safe zone} of a guard $g_i$, denoted by $A(i)$, is the set of triangles for which $h_i$ is an edge (in later figures, each diagonal will be labeled as $g_i$ instead of $h_i$). Since $g_i$ is always at the boundary of the triangles in $A(i)$, it is clear that all triangles in $A(i)$ are safe triangles.
\item Unsafe Zone: An \textit{unsafe zone} of an endpoint $v_j(i)$ is the set of unsafe triangles incident\footnote{We say that a triangle is incident to an endpoint if the endpoint is a vertex of the triangle.} to $v_j(i)$ and it is denoted by $U_j(i)$.
\end{enumerate}
\end{definition}
\begin{figure}[htb]
	\begin{center} 
		\includegraphics[width=0.55\linewidth,height=0.34\linewidth]{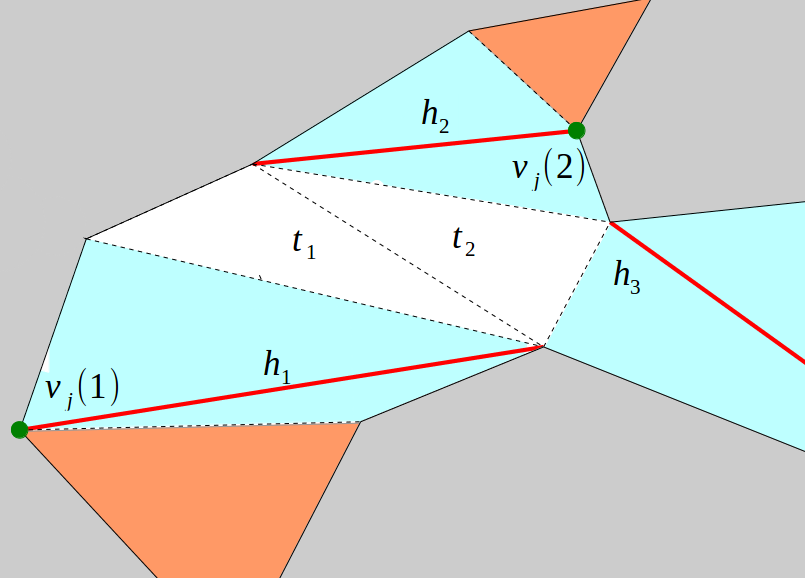}
	\end{center}
	\caption{Classification of the triangles obtained from the triangulation of $P$. 
	}
	\label{fig:triangles}
\end{figure}
\begin{definition}{Neighboring guard:}
	\label{def:2}
	Let $T_j(i)$ be the set of triangles incident to endpoint $v_j(i)$. We say that the set of triangles $T(i) = T_1(i) \cup T_2(i)$ is ``incident'' to $h_i$. A guard $g_k \in S_g\backslash \{g_i\}$ is said to be a neighbor of $g_i$ if $T(i) \cap T(k) \neq \emptyset$.
\end{definition}

\section{Description and Identification of the Guards}
\label{sec:procedure}

In this section we classify the guards based on the type of triangles defined in Section \ref{sec:intro} that are incident to the diagonals of the guards. We also introduce the notion of {\it critical curves} of a guard, which act as a trigger for the guard to implement its reactive strategy when an intruder crosses them. We use the critical curves to define a \textit{critical region} of a guard which determines its location along its diagonal as a function of the location of the intruder when it is inside the critical region. Although the construction of critical region differs for each type of guard, the reactive strategy follows the same idea.

\subsection{Type $0$ Guards}
\label{subsec:zero}
\begin{figure}[thpb]
	\begin{center} 
		\includegraphics[width=0.4\linewidth,height=0.32\linewidth]{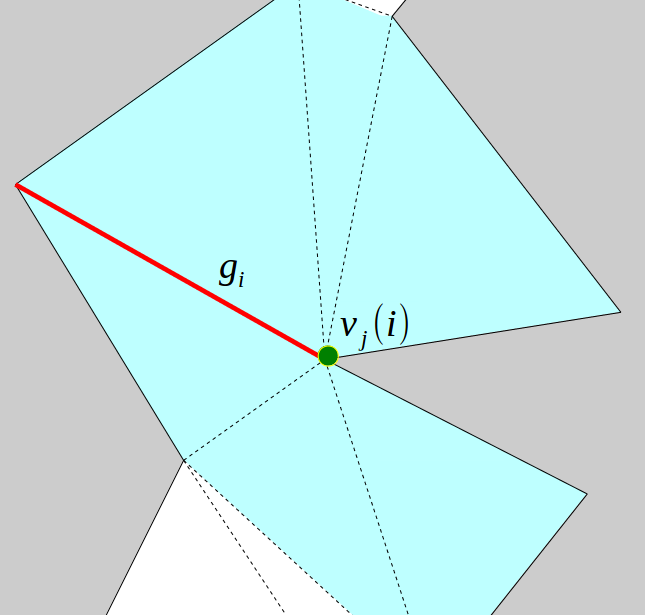} 
	\end{center}
	\caption{Since all the triangles in $T(i)$ are safe when $g_i$ is located at $v_j(i)$, it is a type $0$ guard.}
	\label{fig:type0}
 \end{figure}
\noindent
A guard $g_i \in S_g$ is of type $0$ if there exists at least one value of $j\in\{1,2\}$ such that for each triangle $t_n\in T_j(i) \backslash A(i)$, $t_n$ is always covered by a guard when $g_i$ is located at the opposite endpoint $v_k(i)$. Type $0$ guards are static. If all the triangles in $T(i)$ are safe triangles, then $g_i$ can be stationed at any point on $h_i$. However, for $j,k\in \{1,2\}$, if all the triangles in $T_j(i)$ are safe triangles, and at least one triangle in $T_k(i)$ is not safe, where $k \neq j$, the guard must remain static at $v_k(i)$. A type $0$ guard is shown in Figure \ref{fig:type0}. It is located at $v_j(i)$ where it can cover all the triangles in $T(i)$.

\subsection{Type $1$ Guards}
\label{subsec:one}

 A guard $g_i$ (not of type $0$) is of type $1$ if it has an unsafe zone $U_j(i)$, and there are no regular triangles in $T_j(i)$ adjacent\footnote{We say that two triangles are adjacent if they have an edge in common.} to any triangle in $B(i)$, where $B(i)$ is the \textit{augmented safe zone} of $g_i$. $B(i)$ consists of all the triangles in $A(i)$ and all safe triangles in $T(i)$ adjacent to $A(i)$. We define an internal critical curve of a type $1$ guard as follows: $s_{int}^j(i)$ is the internal critical curve of guard $g_i$ at endpoint $v_j(i)$. It is the boundary between $U_j(i)$ and $B(i)$. In general, $s_{int}^j(i)$ consists of two connected line segments. We define $R= \bigcup_{t_n \in U_j(i)} t_n$ as a region of the environment that can only be covered by $g_i$. $s_{int}^j(i)$ partitions $P$ into two regions, one of them containing $R$, it is called $P_R$. The external critical curve, $s_{ext}^j(i)$, is a curve inside $P \backslash P_R$ which is at a constant distance of $d_M^i$ from $s_{int}^j(i)$. In general $s_{ext}^j(i)$ consists of a set of connected line segments and arcs of circle. The critical region $C_j(i)$ is defined as the region enclosed by $s_{int}^{j}(i)$ and $s_{ext}^{j}(i)$. A type $1$ guard is illustrated in Figure \ref{fig:type1}. The two connected blue segments in the figure represent $s_{int}^j(i)$ , and the connected set of black arcs represents $s_{ext}^j(i)$. The importance of a critical region is that we can determine the location of $g_i$ along its diagonal based on the location of $I$ when it is inside the critical region, such that when $I \in s_{int}^j(i)$, $g_i$ is located at $v_j(i)$, and when $I \in s_{ext}^j(i)$, $g_i$ is located at $v_k(i)$ with $k\neq j \in \{ 1, 2 \}$. This is detailed in Subsection \ref{subsec:reactive}.

\begin{figure}[htb]
	\begin{center} 
		\includegraphics[width=0.6\linewidth,height=0.32\linewidth]{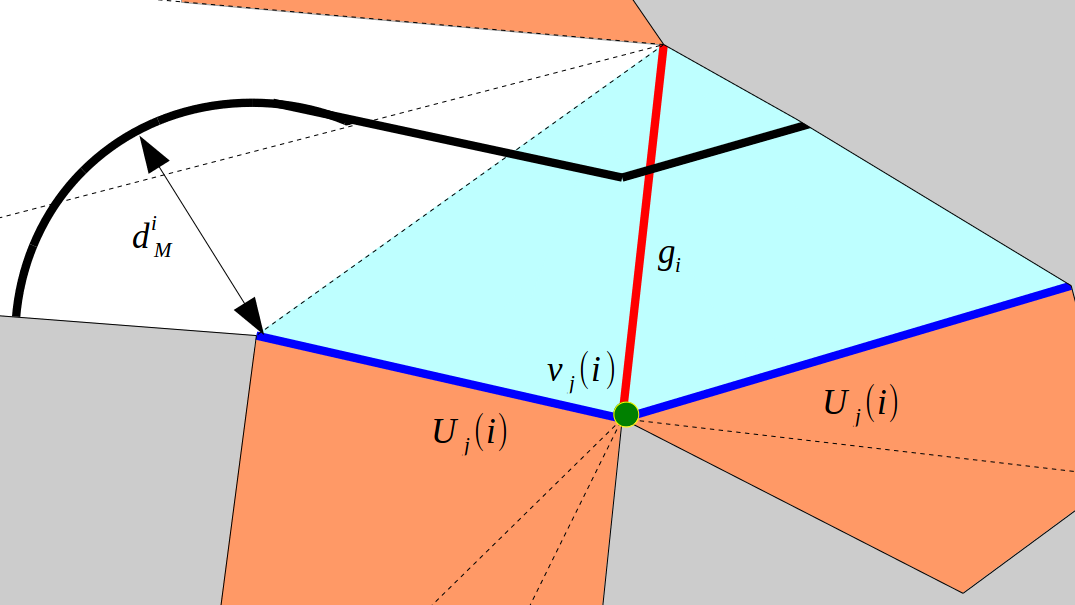} 
	\end{center}
	\caption{In this case $B(i)=A(i)$. Since there are no regular triangles in $T_j(i)$ adjacent to $B(i)$, $g_i$ is type $1$. Its critical region is enclosed by the black curve and blue segments.}
	\label{fig:type1}
\end{figure}

\subsection{Type $2$ Guards}
\label{subsec:two}

A guard $g_i \in S_g$ (not of type $0$ or type $1$) is of type $2$ if all the neighboring guards that can cover the regular triangles in $T_1(i)$ or $T_2(i)$ have their critical regions defined. Only type $1$ and type $2$ guards have critical regions. Hence, to identify a type $2$ guard all of its neighbors must be type $1$ guards or type $2$ guards with their regions already defined. To define the critical region of a type $2$ guard we start defining $R_j(i) \subset T_j(i)$ as the set of regular triangles incident to $v_j(i)$, and we also define $N_j(i)$ as the set of all neighboring guards that cover the triangles in $R_j(i)$.

For each $t_n \in R_j(i)$, we define $S_{t_n} \subset N_j(i)$ as the set of guards that can cover $t_n$. For each guard $g_l \in S_{t_n}$, we know that if $p_I \in C_j(l)$, then $p_{g_l} \notin v_j(l)$. Let $B_n = t_n \cap (\bigcap_{g_l \in S_{t_n}}{C_j(l)})$. If $B_n \neq \emptyset$ and if $p_I \in B_n$, then $g_i$ is the only guard that can cover $B_n$. We define a curve $s_{int}^{j,n}(i)$ as the boundary of each $B_n$. There might be cases in which some of the regions $B_n$ are adjacent. If $B_n$ is not adjacent to any $B_m$ with $m \neq n$, then we define $B_w^j=B_n$. Otherwise, let $S_B$ be the region obtained from the union of all adjacent regions $B_n$, then $B_w^j= S_B$. After obtaining all $B_w^j$, a curve $s_{int}^{j,w}(i)$ is defined as the boundary of each $B_w^j$. Additionally, if $g_i$ has an unsafe zone $U_j(i)$, an additional $s_{int}^{j,w}(i)$ that corresponds to $U_j(i)$ is obtained (Please refer to subsection \ref{subsec:one}). Each $s_{int}^{j,w}(i)$ partitions $P$ into two regions. Let $P_R^j$ be the region containing $B_w^j$. We define $R= \bigcup_{\forall w} {B_w^j}$. For each $s_{int}^{j,w}(i)$, a curve $s_{ext}^{j,w}(i)$ inside $P \backslash P_R^j$ is generated, it is at a distance of $d_M^i$ from $s_{int}^{j,w}(i)$. $c_j^w(i)$ is the region enclosed between $s_{int}^{j,w}(i)$ and $s_{ext}^{j,w}(i)$. The critical region is defined as $C_j(i)= \bigcup_{w} c_j^w(i)$. In Figure \ref{fig:type4}, $g_3$ is a type $2$ guard, and $g_1$ and $g_2$ are of type $1$. $s_{int}^j(1)$ is the boundary between sets $A(1)$ and $U_j(1)$. $s_{int}^j(2)$ is defined in a similar manner. $s_{int}^{j}(i)$ and $s_{ext}^{j}(i)$ are shown as blue segments and black curves respectively.

\begin{figure}[htb]
	\begin{center} 
		\includegraphics[width=0.5\linewidth,height=0.3\linewidth]{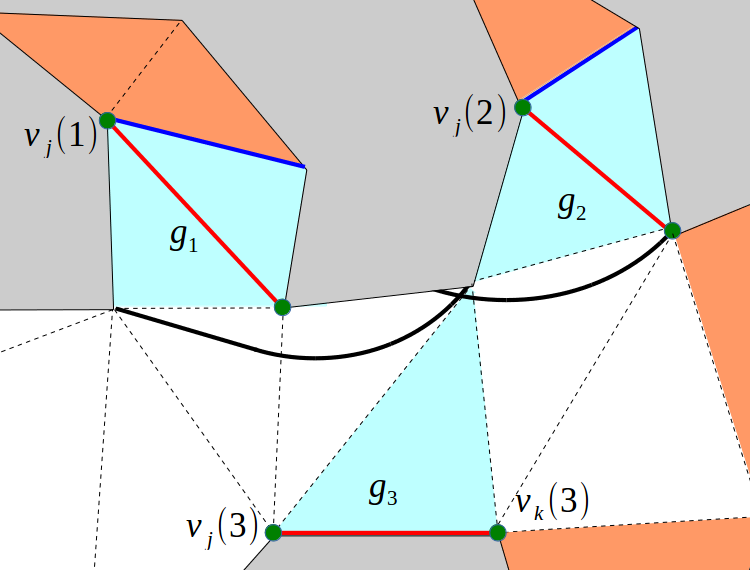} 
	\end{center}
	\caption{A type $1$ guard $g_2$ is illustrated. It is the only guard in $N_k(3)$. Since $C_j(2)$ is defined, $g_3$ meets the definition of a type $2$ guard.}
	\label{fig:type4}
\end{figure}

\subsection{Type $3$ Guards}
\label{subsec:three}

A guard $g_i \in S_g$ is a type $3$ guard if it is not of type $0$,$1$ or $2$. There is at least one regular triangle in $T_j(i)$ and one in $T_k(i)$ which are adjacent to $B(i)$, such that not all the neighbors that can cover those triangles have their critical regions defined. Hence, we cannot determine all the regions $R$ that define the internal critical curve of type $3$ guards. We proceed by transforming all the regular triangles in $R_j(i)$ (or $R_k(i)$) into unsafe triangles. This turns $g_i$ into a type $1$ guard. Consequently, $g_i$ has the task of covering all the triangles in $T_j(i)$(or $T_k(i)$). Thus, all the triangles in $R_j(i)$ are ``safe triangles'' for the neighboring guards. Hence, some guards in $N_j(i)$ can become type $0$, type $1$ or type $2$ guards. In Figure \ref{fig:type5_2}, two type $3$ guards $g_1$ and $g_2$ are illustrated. $t_1$ was originally a regular triangle. After converting $g_1$ into a type $1$ guard, $t_1$ becomes an unsafe triangle that must be covered by $g_1$. Thus $t_1$ is considered a ``safe triangle'' for $g_2$, so it also becomes a type $1$ guard.

\begin{figure}[htb]
	\begin{center} 
		\includegraphics[width=0.53\linewidth,height=0.46\linewidth]{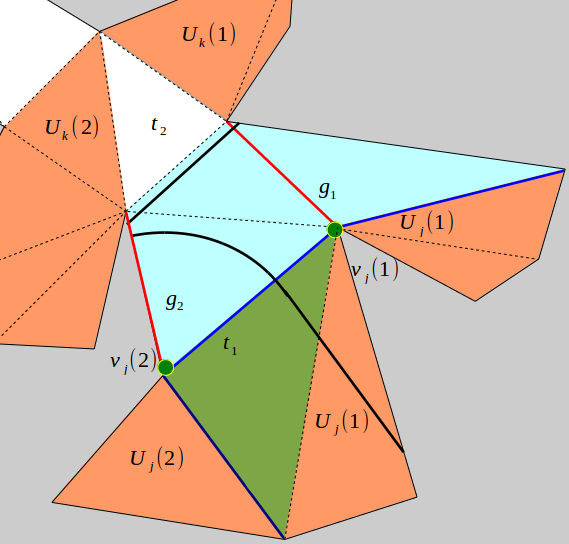} 
	\end{center}
	\caption{$t_1$ is turned into an unsafe triangle covered by $g_1$. Thus, $g_1$ and $g_2$ become type $1$.}
	\label{fig:type5_2}
\end{figure}

\subsection{Reactive Strategy for Type $1$ and Type $2$ Guards}
\label{subsec:reactive}

Let $p_I \in C_j(i)$. If $g_i$ is a type $1$ guard, $d_{min}^j(i)$ is the minimum distance between $p_I$ and $s_{ext}^{j}(i)$. Otherwise, if $g_i$ is a type $2$ guard, $d_{min}^j(i) = \max\limits_{\forall w} d_{min}^{j,w}(i)$, where each $d_{min}^{j,w}(i)$ is the minimum distance between $p_I$ and $s_{ext}^{j,w}(i)$. The following equation maps the location of the intruder to obtain the location of a guard, $p_{g_i}$, along its diagonal:	
\begin{equation}
\label{eq:1}
p_{g_i}= p_{v_k(i)}+ \frac{d_{min}^j(i)}{d_M^i}(p_{v_j(i)}-p_{v_k(i)}),
\end{equation}
where $p_{v_j(i)}$ and $p_{v_k(i)}$ are the coordinates of $v_j(i)$ and $v_k(i)$ respectively. If $p_I \notin C_j(i)$ we have two cases: if $p_I \in P_R$, $g_i$ remains static at $v_j(i)$, otherwise, it stays at $v_k(i)$. Lemma \ref{lemma:1} gives a sufficient condition for a guard to cover all the triangles incident to its diagonal.

\begin{lemma} 
	\label{lemma:1}
	Let $g_i$ be a type $1$ or type $2$ guard. The motion strategy induced by $C_j(i)$ guarantees that $I$ will be visible to $g_i$ when $p_I \in R$.
\end{lemma}
\begin{proof}
		If $p_I \notin C_j(i)$, $g_i$ will remain static at $v_j(i)$ if $p_I \in P_R$ or $v_k(i)$, otherwise. If $p_I \in P \backslash R$, and $I$ moves towards $R$, then $g_i$ remains static at $v_j(i)$ (or $v_k(i)$ if $p_I \in P \backslash P_R$) until $p_I \in C_j(i)$. Thus, $g_i$ starts its motion leaving the endpoint where it is located. The distance between $s_{int}^j(i)$ and $s_{ext}^j(i)$ is $d_M^i$, which is the maximum distance that $I$ can travel while $g_i$ moves from $v_k(i)$ to $v_j(i)$. This guarantees that regardless of the motion of $I$, once that it arrives to $s_{int}^j(i)$, $g_i$ is located in $v_j(i)$ so it is covering the set of triangles that contain $R$. Hence, $I$ is visible to $g_i$.
\end{proof}

\section{Modeling the Problem as an Hybrid System}
\label{sec:overall}

In this section we formulate the tracking problem as an hybrid system. An hybrid system is a mathematical model that is able to describe the evolution of continuous dynamics and also the discrete switching logic, which can include uncertainty in both the continuous and discrete input variables \cite{Tomlin:2000}. An hybrid automaton is a collection $H=(Q,X,\Sigma,V, \mbox{Init}, f, \mbox{Inv}, R_t)$ where $Q \cup X$ is a finite collection of state variables, $Q$ is the set of discrete states, called modes, and $X=\mathbb{R}^n$ is the set of continuous states. The state of $H$ is $(q,x) \in Q \times X$. $\Sigma=\Sigma_1 \cup \Sigma_2$ is a finite collection of discrete input variables, $\Sigma_1$ is the set of discrete control inputs and $\Sigma_2$ the set of discrete disturbances. $V=U \cup D$ is the set of continuous input variables, $U$ and $D$ are the sets of continuous inputs and disturbances respectively. The input of $H$ is $(\sigma,v) \in \Sigma \times V$. $\mbox{Init} \subseteq Q \times X$ is the set of initial states. $f: Q \times X \times V \rightarrow X$ is a map from the set that consists of the system state and the continuous inputs to the set of continuous states. It describes the continuous evolution of $x \in X$ in each $q \in Q$. This function is assumed to be globally Lipschitz in $X$ and continuous in $V$. $\mbox{Inv} \subseteq Q \times X \times \Sigma \times V$ is the invariant set, it includes those states and inputs for which continuous evolution is allowed. Finally, $R_t: Q \times X \times \Sigma \times V \rightarrow 2^{Q \times X}$ is a map that encodes the discrete transitions from one mode to other. 

In our problem, depending on the location of $I$ some guards might remain static while others might require to move. We partition the environment using the critical curves of the guards and their intersections. We define $S_R$ as the set of regions obtained in the partition. Let $R_{ext}= \overline{\bigcup\limits_{\forall g_i \in S_g}{C_j(i)}} \in S_R$ be the external region, it includes all the locations that are not inside any critical region. When $I$ is inside any region $R_j \in S_R$, a specific subset of the guards is required to move (if $I \in R_{ext}$, all the guards remain motionless). There is a different continuous control for each one of the regions. Consequently, we define a mode $q_j$ for each $R_j \in S_R$. The collection of all $q_j$ is the set $Q$. In Fig. \ref{fig:regions} a) an environment with the regions generated by the intersection of critical regions is illustrated. $R_{ext}$ is shown as a blue area. $X$ contains all the combinations of possible locations of $I$ and the guards. Since each $g_i$ is constrained to move along $h_i$, $X=h_1 \times h_2 \times \ldots \times h_m \times P$, where $1 \leq m \leq \rfloor \frac{n}{4} \lfloor$ is the number of deployed guards, and $P$ is the environment which represents the set of all locations of $I$. We have $\Sigma_1$ since there is a discrete control that triggers the motion of different sets of guards when $I$ transitions from one region to another. There are not discrete disturbances, so $\Sigma_2= \emptyset$. $\Sigma= \Sigma_1$ is the set of discrete variables $\sigma_j$. There is a $\sigma_j$ for each $R_j \in S_R$. We define $\sigma_j$ as a discrete variable such that $\sigma_j=j$ when $p_I \in R_j$. We are interested in the transitions between modes. Transition happens when $I$ crosses the boundary $\partial R_{j,k}$ between adjacent regions $R_j$ and $R_{k}$. Thus, $\sigma_j \in \{j, \mathbb{I}_j \}$ where $\mathbb{I}_j$ is the set of indexes of the regions that are adjacent to $R_j$. Thus, if $I \in \partial R_{j,k}$ then $\sigma_j=k$ triggers the transition from mode $q_j$ to $q_k$.

\begin{figure}[thpb]
	\begin{center}
		\subfigure[]{\includegraphics[width=0.45\linewidth,height=0.45\linewidth]{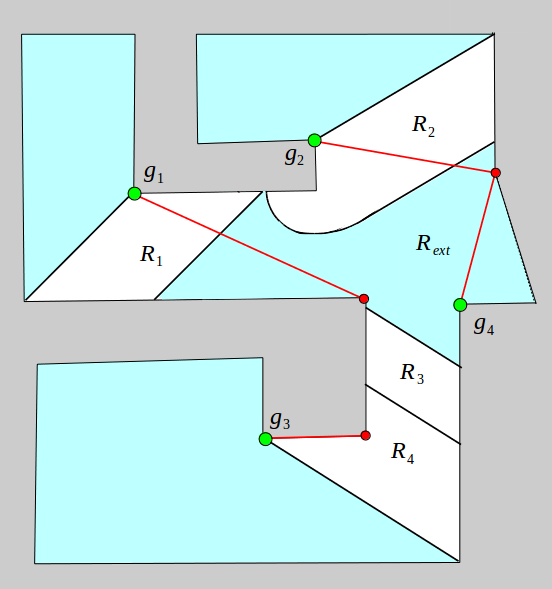}}
		\subfigure[]{\includegraphics[width=0.7\linewidth,height=0.5\linewidth]{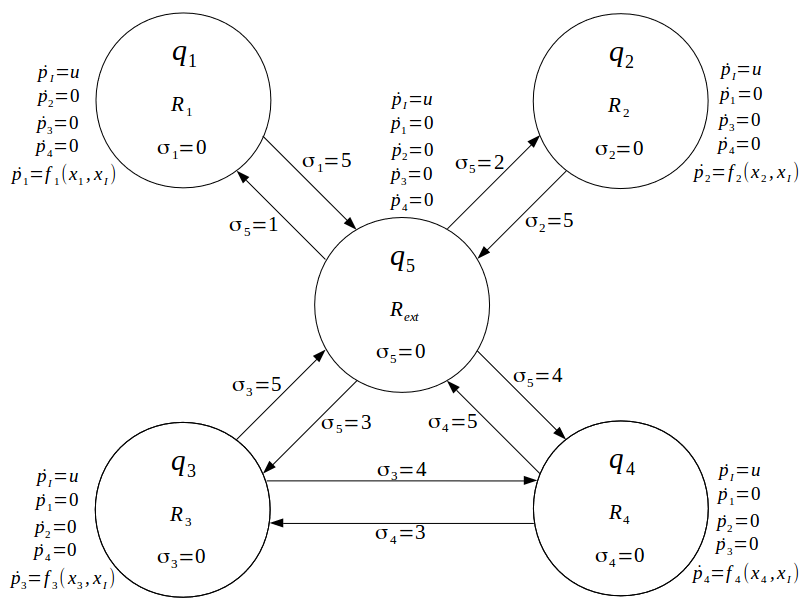}}
		\end{center}
		\caption{a) The partition of the environment obtained from the critical regions is illustrated. b) The corresponding hybrid system is shown.}
	\label{fig:regions}
\end{figure}

Since the continuous control described in (\ref{eq:1}) is purely reactive, $U= \emptyset$. The continuous dynamics of the system depends on $p_I$, then $V=D$. Depending on $p_I(0)$, the guards are located in a specific location along their diagonals. The set of those states is the init set. $f$ is defined by (\ref{eq:1}). For each guard that is required to move when $p_I \in R_j$, (\ref{eq:1}) determines its location along $h_i$. Since $\mbox{Inv}$ consists of those states and inputs $(q,x,\sigma_1,d)$ for which the continuous evolution of $H$ is allowed (there is not a discrete transition), the invariant set consists of all the states where $I \in R_j \backslash \partial R_{j,k}$. $R_t$ is defined by the adjacency between regions $R_j$ and the discrete controls that trigger the change between modes.

Consider the example illustrated in Figure \ref{fig:regions} a). The diagonals of the guards, and the partition of $P$ obtained from the critical regions are shown. The hybrid automaton representing the problem is illustrated in Figure \ref{fig:regions} b). The components of the automaton for this example are the following: $Q={q_1,q_2,q_3,q_4,q_5}$ where each $q_i$ corresponds to $R_i$ with $i \in \{1,2,3,4\}$ and $q_5$ corresponding to $R_{ext}$. $X=h_1 \times h_2 \times h_3 \times h_4 \times P$. $\Sigma= \Sigma_1 = \{\sigma_1,\sigma_2,\sigma_3,\sigma_4,\sigma_5 \}$, where $\sigma_1 \in \{1,5\}$, $\sigma_2 \in \{2,5\}$, $\sigma_3 \in \{3,4,5\}$, $\sigma_4 \in \{4,3,5\}$, and $\sigma_5 \in \{5,1,2,3,4\}$. $V=D=P$, where the interior of the polygon $P$ represents the set of all possible locations of $I$. From (\ref{eq:1}), it is clear that we have a continuous function for each guard $g_i$, $f_i: p_{gi} = p_{vk}(i) + \frac{d_{min}^j(i)}{d_M^i} (p_{vj}(i)-p_{vk}(i))$ with $i \in \{1,2,3,4\}$. Init set includes all the states where the guards are located at their corresponding locations according to (\ref{eq:1}). The discrete transition mapping $R_t$ is represented in Figure \ref{fig:regions} b). The invariant set, Inv, is the set that includes all the states where (\ref{eq:1}) can be used to describe the evolution of the system, which means that it does not include those states where $I$ is located at the boundary between any two regions.

\section{Reachability Analysis of the Hybrid System}
\label{sec:reachability}

An important concept in hybrid systems is \textit{trajectory}, $\tau=(I_i)_{i=0}^N$. It is a finite or infinite sequence of intervals of the real line such that $I_i=[ \tau_i, \tau'_i ]$ for $i<N$, and if $N< \infty$, $I_N=[\tau_N, \tau'_N]$ or $I_N=[\tau_N, \tau'_N)$. Also, for all $i$, $\tau_i \leq \tau'_i=\tau_{i+1}$.
Each $\tau_i$ is a time at which there is a transition between modes, and all the time inside each interval is continuous. Since the transitions are assumed to be instantaneous, then multiple discrete transitions can take place at the same time. An \textit{execution} of an hybrid automaton is an hybrid trajectory $\chi=(\tau,q,x,\sigma,v)$ such that $(q(\tau_0),x(\tau_0)) \in \mbox{Init}$. For the continuous evolution $q(\cdot)$, $\sigma(\cdot)$ must be constant, $v(\cdot)$ is piecewise continuous and $f$ must describe the change of the continuous state $x(t)$ for all $t \in [\tau_i, \tau'_i)$, and $(q(t),x(t),\sigma(t),v(t)) \in \mbox{Inv}$. Finally, for the discrete evolution, $(q(\tau_{i+1}),x(\tau_{i+1})) \in R(q(\tau'_{i}),x(\tau'_{i}),\sigma(\tau'_{i}),v(\tau'_{i}))$ must hold for all $i$.

A \textit{trajectory acceptance condition} is defined as an arbitrary specification that the execution of the system must satisfy. We decide to define a specification in terms of the safety. The safe set $F \subseteq Q \times X$ is a subset of the state space in which the system is defined to be safe. In our case, it implies that $I$ is visible to at least one guard in any state $(q,x) \in F$. The unsafe set is defined as $G=F^c$. Given the critical regions of the guards and the hybrid automaton, we can determine if it is possible to keep track of $I$ all the time through a reachability analysis. We must ensure that at any time, regardless of the motion of $I$ the system remains inside a safe set. Otherwise, at least one additional guard would be required. The reachability analysis allows us to determine the \textit{maximal controlled invariant set} of a safe set $F$, which is defined as the maximal subset of $F$ such that there exists a controller that guarantees that if any execution starts in the subset, the execution stays in the subset for all future time. Which implies that as long as the initial state of the system belongs to $F$ then $I$ will be visible to at least one guard all the time. The controllable predecessor sets $Pre_1$ and $Pre_2$ of a given set $K \subseteq Q \times X$ are defined as follows: $Pre_1(K)=\{ (q,x) \in K: \exists (\sigma_1,u) \in \Sigma_1 \times U \forall (\sigma_2,d) \in \Sigma_2 \times D, (q,x,\sigma_1,\sigma_2,u,d) \notin \mbox{Inv} \wedge R(q,x,\sigma_1,\sigma_2,u,d) \subseteq K  \}$, and $Pre_2(K^c)=\{ (q,x) \in K: \forall (\sigma_1,u) \in \Sigma_1 \times U \exists (\sigma_2,d) \in \Sigma_2 \times D, R(q,x,\sigma_1,\sigma_2,u,d) \cap K^c \neq \emptyset  \} \cup K^c $.

As described in \cite{Tomlin:2000}, $Pre_1$ consists of all the states in $K$ for which controllable actions can force the system to remain in $K$ while there is a discrete transition. In contrast, $Pre_2$ contains all the states in $K^c$ and all states in $K$ where disturbances are able to force the system outside $K$. Another important concept is the \textit{Reach} operator. Given $G,E \subseteq Q \times X$ such that $G \cap E = \emptyset$, $Reach(G,E)= \{ (q,x) \in Q \times X| \forall u \in U \exists d \in D \mbox{ and } t \geq 0 \mbox{ such that } ((q(t),x(t)) \in G \mbox{ and } (q(s),x(s)) \in \prod{(Inv) \backslash E} \mbox{ for } s \in [0,t] \}$, where $(q(s),x(s))$ is the continuous state trajectory starting at $(q,x)$, and $\prod{(Inv)}$ represents the state space components of Inv. $Reach(G,E)$ contains those states that belong to $G$ and also states in $Q \times X \backslash G$ from which, for all controls $u(\cdot) \in U$ there is a disturbance $d(\cdot) \in D$ such that the state trajectory $(q(s),x(s)$ drives the system to $G$ while avoiding $E$. Using the Maximal Controlled Invariant Set algorithm presented in \cite{Tomlin:2000} the maximal controlled invariant set is obtained.

\begin{algorithm}
\caption{Maximal Controlled Invariant Set}
\label{alg:maximal}
\begin{algorithmic}
\State \textbf{Input}: $W^0=F$, $W^{-1}=\emptyset$, $i=0$
\State \textbf{Output}: $W^*$
\While{ $W^i \neq W^{i-1}$}
\State 1. $i \leftarrow i-1$
\State 2. $W^{i-1} \leftarrow W^i \backslash Reach(Pre_2((W^i)^c),Pre_1(W^i))$
\EndWhile
\State 3. $W^*=W^i$
\end{algorithmic}
\end{algorithm}

For the tracking problem, $Pre_1(W^0)=Pre_1(F)$ consists of all the states where there is at least one guard covering the triangle at which $I$ is located, such that there would be at least one guard covering the triangle where $I$ is located after a transition from one region $R_j$ to a region $R_k$. Since the number of guards covering the triangle where $I$ is located and their location do not change after a transition between modes (the location of $I$ does not change when the transition happens), then if $I$ is visible to a guard it will be visible after a discrete transition. Thus, $Pre_1(W^0)$ consists of all the states where $p(I) \in \partial R_{j,k}$ for any $j \neq k$, such that there is at least one guard covering the triangle where $I$ is located. Hence, all states where $p_I$ is in the interior of any $R_j$ do not belong to $Pre_1(W^0)$, neither the states where the location of the guards is such that they are not covering the triangle where $I$ is located.

$Pre_2((W^0)^c)=Pre_2(F^c)$ consists of all the states where the triangle at which $I$ is located is not covered by any guard, and also of all safe states where $p_I \in \partial R_{j,k}$ such that after a discrete transition, the triangle at which $I$ is located will not be covered by any guard. As mentioned in the definition of $Pre_1(W^0)$, if there was at least one guard covering the triangle where $I$ is located before a discrete transition, it will not change after the transition takes place. Therefore, $Pre_2((W^0)^c)$ consists only of the unsafe states.

$Reach(Pre_2((W^0)^c),Pre_1(W^0))$, consists of all the unsafe states, and also of all the safe states from which $I$ can follow a trajectory that leads to a state in $Pre_2((W^0)^c)$ while avoiding any state in $Pre_1(W^0)$. Consequently, $Reach(Pre_2((W^0)^c),Pre_1(W^0))$ consists of all states where $I$ is located at the interior of any $R_j$ such that there is at least one unsafe state with the intruder's location inside $R_j$. Since there is no control $u$ that can prevent $I$ to move towards the unsafe state in the interior of $R_j$, then every state in which $I$ is in the interior of $R_j$ belongs to $Reach(Pre_2((W^0)^c),Pre_1(W^0))$. Hence, $Reach(Pre_2((W^0)^c),Pre_1(W^0))$ includes all the states where $I$ is located inside a region $R_j$ with an unsafe state. We call those $R_j$ regions as \textit{forbidden}. Therefore, $W^{-1} \leftarrow W^0 \backslash Reach(Pre_2((W^0)^c),Pre_1(W^0))$ consists of those safe states where the intruder is located in regions $R_j$ that are not forbidden.

In the second iteration, $Pre_1(W^{-1})$ consists of all states where $p(I) \in \partial R_{j,k}$ for any $j \neq k$, such that there is at least one guard covering the triangle where $I$ is located and $R_k$ is not a forbidden region. $Pre_2((W^{-1})^c)$ consists of all the unsafe states and all the safe states such that $p(I) \in \partial R_{j,k}$ and $R_k$ is a forbidden region. Consequently, $Reach(Pre_2((W^{-1})^c),Pre_1(W^{-1}))$ includes all the states from which the intruder can reach the boundary $\partial R_{j,k}$ between any region $R_j$ and a forbidden region $R_k$. This includes all the states where $I$ is located at the interior of any region $R_j$ that shares a boundary with $R_k$. From this point it is clear that the execution of the algorithm will eventually lead to the definition of a forbidden region that is equal to the whole environment, and therefore $W^*$ is the empty set. Consequently, if there is at least one location $p_{unsafe}$ inside the environment such that there is not any guard covering the triangle where the intruder is located when $p_I=p_{unsafe}$, then there is no guarantee that the intruder can be tracked all the time by the set of guards deployed in the environment. Otherwise, if there is not such location $p_{unsafe}$, then is guaranteed that the set of guards deployed in the environment, following the strategy given by their critical regions will always keep track of $I$.

Given the definition of the critical regions in Section \ref{sec:procedure} and the definition of the different types of triangles. Unsafe states can only exists on regular and unsafe triangles. The following results establish the existence of $p_{unsafe}$ points that can lead to an empty maximal controlled invariant set. First, we have an unsafe state if guard $g_i$ has unsafe zones $U_j(i)$ and $U_k(i)$ and its critical region $C_j(i)$ intersects with $U_k(i)$. According to (\ref{eq:1}) it is clear that if the intruder follows the shortest path between $U_j(i)$ and $U_k(i)$, when it reaches $U_k(i)$, there will not be any guard covering the triangle where $I$ is located. So every state that includes that has a point inside $C_j(i) \cap U_k(i)$ is a $p_{unsafe}$ point. Also, we have the following result regarding regular triangles.

\begin{lemma} 
	\label{lemma:3}
	A regular triangle $t$ that is covered by a set of guards $S_t \subset S_g$, does not contain $p_{unsafe}$ points if and only if $\bigcap_{g_i \in S_t} C_j(i) \cap t$ is an empty set or a single point.
\end{lemma}
\begin{proof}
	Assume $p_{g_i} = p_{v_k(i)}$, $\forall \: g_i \in S_t$ , $t$ is a safe triangle, and $\bigcap_{g_i \in S_t} C_j(i) \cap t$ contains more than one point. When $p_I \in (\bigcap_{g_i \in S_t} C_j(i) \cap t)$, $p_{g_i} \neq p_{v_k(i)}$ $\forall g_i \in S_t$, then all the guards $g_i \in S_t$ leave the endpoints from which they covered $t$ according to (\ref{eq:1}), so $t$ becomes uncovered. Since visibility of the intruder is not guaranteed, $t$ is not a safe triangle, which is a contradiction. Conversely, if $t$ is not a safe triangle and $\bigcap_{g_i \in S_t} C_j(i) \cap t$ contains at most one point, there does not exist a region inside $t$ where the presence of the intruder inside it causes all $g_i$ to leave $v_k(i)$. Hence $t$ is covered by at least one guard. Therefore, any time that $p_I \in t$ there is at least one guard covering $t$ so there cannot be $p_{unsafe}$ points in $t$.
\end{proof}

\section{Computation of Maximum Speed Ratio}
\label{sec:speed}

We know that the tracking strategy works if the condition in Lemma \ref{lemma:3} is satisfied. However, if the condition fails it means that there is at least one unsafe state that can be reached when the intruder is located at a specific point inside the environment. In this section, we present a method to obtain an upper bound for the speed ratio $r$ that ensures that the guards can track the intruder all the time. To determine this upper bound, we consider the two cases that indicate the necessity of additional guards to track the intruder. 

For the first case, let $g_i \in S_g$ such that $U_j(i),U_k(i) \neq \emptyset$. If $d_{min}^j(i)<d_M^i$, where $d_{min}^j(i)$ is the minimum distance between $U_j(i)$ and $U_k(i)$, then $g_i$ is incapable of moving from $v_j(i)$ to $v_k(i)$ while the intruder follows the shortest path from $U_j(i)$ to $U_k(i)$. Hence, for each $g_i$, we require that $d_M^i=r l_i \leq d_{min}^j(i)$. Therefore, we have that the maximum speed ratio is $r'= \min\limits_{\forall g_i | U_j(i),U_k(i) \neq \emptyset} \left\{ \frac{d_{min}(i)}{l_i} \right\}$. The second case is when the condition of Lemma \ref{lemma:3} is not met, so there exists at least one point $p_{unsafe}$ in a regular triangle. This implies that $\bigcap_{g_i \in S_t} C_j(i) \cap t$ contains more than one point. Since each critical region is a function of $r$, there exists a maximum speed ratio $r_n$ such that if $r=r_n$, $\bigcap_{g_i \in S_t} C_j(i) \cap t$ is a single point (and for $r< r_n$ it is an empty set).

For each regular triangle $t_n$, we consider the guards $g_i \in S_{t_n}$. We know that $C_j(i)$ grows when $r$ increases and new intersections between critical regions appear. So $r$ can be increased until $\bigcap_{g_i \in S_{t_n}} C_j(n) \cap t_n$ contains more than one point. We start by identifying the first two critical regions that intersect, namely $C_j(i^1)$ and $C_j(i^2)$ with $g_{i^1} \neq g_{i^2} \in S_{t_n}$. Let $r_{i^1,i^2}$ be the speed ratio such that $C_j(i^1) \cap C_j(i^2)$ is a single point $q_{i1}^{i2}$. By the definition of a critical region, we know that $q_{i1}^{i2} \in p_{i1}^{i2}$, where $p_{i1}^{i2} \subset P$ is the path of minimum distance from any $s_{int}^{j,w}(i^1)$ to any $s_{int}^{j,y}(i^2)$. Since each external critical curve is a function of $r$, $r_{i^1,i^2}=\frac{l(p_{i1}^{i2})}{l_{i^1}+l_{i^2}}$, where $l(p_{i1}^{i2})$ is the length of $p_{i1}^{i2}$. Notice that if $g_i$ is type $2$, there is at least one $s_{int}^{j,y}(i)$ that depends on $r$, so $l(p_i^m)$ may not be a constant but a function of $r$. The intersection of critical regions $\bigcap_{g_i \in S_{t_n}} C_j(n)$ depends on the intersection of $C_j(i^1)$ and $C_j(i^2)$, so it can be described as a function $f(i^1,i^2,r)$. However, $\bigcap_{g_i \in S_{t_n}} C_j(n)$ changes when a new critical region intersects with $f(i^1,i^2,r)$. Therefore, the strategy finds the speed ratio $r_{i^1,i^2,m}$ for each $g_m \in S_{t_n} \backslash \{g_i^1, g_i^2\}$ that corresponds to the intersection between $f(i^1,i^2,r)$ and $C_j(m)$. The next critical region to be selected corresponds to the minimum $r_{i^1,i^2,m}$. The procedure continues iteratively until $f(r)=\bigcap_{g_i \in S_{t_n}} C_j^w(n)$ is obtained. The intersection point of $\bigcap_{g_i \in S_{t_n}} C_j^w(n)$, called $q_\cap$, now can be obtained. If $q_\cap  \in t_n$ then the value of $r$ is the maximum speed ratio for $t_n$. Otherwise, $r$ is the speed ratio at which $f(r) \cap t_n \neq \emptyset$. The maximum speed ratio that corresponds to $t_n$ is $r_n=r$, and the maximum speed ratio at which the set of guards deployed initially can track the intruder at all times is $r_{max}= \min \{ \min \{r_n \}, r' \}$.

\begin{proposition} 
	\label{theorem:2}
	The motion strategy given by the critical regions guarantees that the initially deployed guards are able to track the intruder all the time if and only if $r \leq r_{max}$.
\end{proposition} 
\begin{proof}
Assume that $r \leq r_{max}$ but there is not guarantee that the motion strategy of the guards ensures that $I$ is tracked all the time. Consequently, there are no $p_{unsafe}$ locations in the environment. Therefore, the reachability analysis of Section \ref{sec:reachability} gives a maximal invariant set that allows any point of the environment to be the location of $I$. Since the maximal invariant set was defined using a safe set consistent of all the states in which there is at least one guard covering the triangle where $I$ is located, the result follows. Now assume that the guards can always cover the triangle where $I$ is located but $r > r_{max}$. Since $r > r_{max}$, then there is at least one location $p_{unsafe}$. Therefore, the reachability analysis yields that the maximal invariant set is empty. Which implies that there is not a state from which the system can start that guarantees that $I$ will be always visible to some guard. This is a contradiction. Therefore, $r \leq r_{max}$.
\end{proof}

\begin{figure}[thpb]
	\begin{center} 
		\includegraphics[height=0.18\textheight,width=0.42\textwidth]{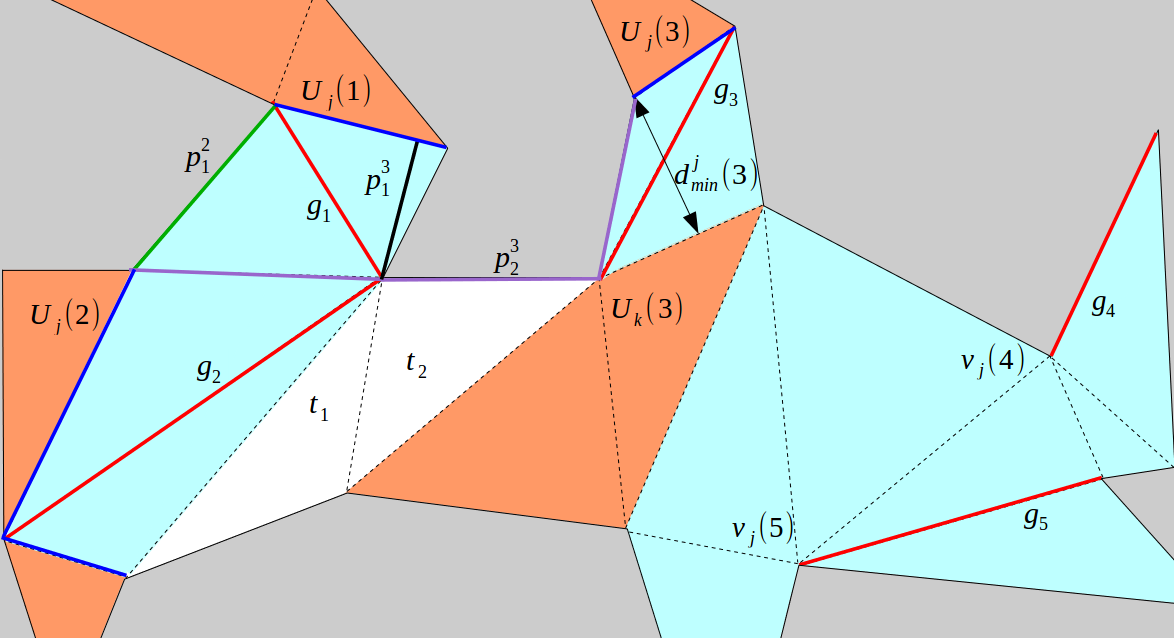} 
	\end{center}
	\caption{Polygon $P$ representing an environment with five guards deployed.}
	\label{fig:example}
\end{figure}

We present an example of the described method. Guards $g_1$, $g_2$, $g_3$, $g_4$ and $g_5$ are deployed in $P$ on red diagonals shown in Figure \ref{fig:example}. $g_4$ is a type $0$ guard since it can cover all the triangles when $p_{g_4}=p_{v_j(4)}$. As a consequence, $g_5$ becomes a type $0$ guard, so we locate it at $v_j(5)$ and all the triangles adjacent to that endpoint become \textit{safe triangles}. For the rest of the guards we illustrate their internal critical curves as blue segments.

\begin{figure}[thpb]
	\begin{center} 
		\includegraphics[height=0.14\textheight,width=0.35\textwidth]{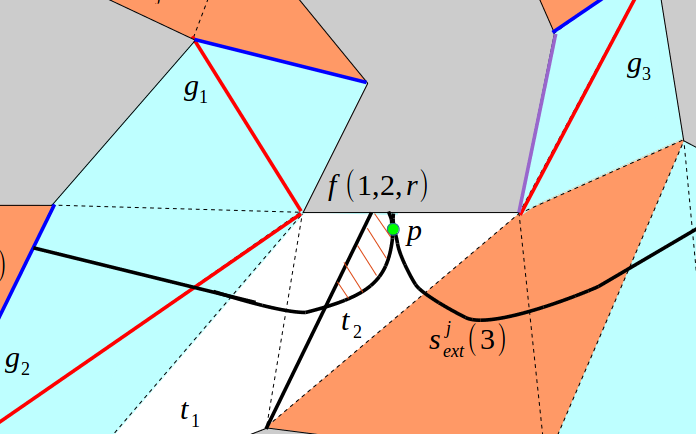} 
	\end{center}
	\caption{Intersection between $f(1,2,r)$ and $s_{ext}^j(3)$ in $t_2$.}
	\label{fig:intersection}
\end{figure}

Since $g_3$ is a type $1$ guard with unsafe zones $U_j(3)$ and $U_k(3)$ and $d_M^3=r l_3$ cannot be greater than $d_{min}^j(3)$, $r'= \frac{d_{min}^j(3)}{l_3}$. Next we consider $t_1$ which can be covered by $g_1$ and $g_2$. We find the path $p_1^2$ between $s_{int}^j(1)$ and $s_{int}^j(2)$ shown as a green segment, then $r_{1,2}=\frac{l(p_1^2)}{l_1+l_2}$. However, $q_1^2 \notin t_1$. Hence, we define a function $f(1,2,r)$ which depends on $s_{ext}^j(1)$ and $s_{ext}^j(2)$. It follows that the maximum speed ratio for $t_1$ is $r_1= \argmin\limits_r d(f(1,2,r),e_1)$. Finally, we consider $t_2$ which can be covered by $g_1$, $g_2$ and $g_3$. Speed ratios $r_{2,3}=\frac{l(p_2^3)}{l_2+l_3}$ and $r_{1,3}=\frac{l(p_1^3)}{l_1+l_3}$ are computed. $p_2^3$ and $p_1^3$ are shown as a magenta and a black trajectory in Figure \ref{fig:example} (in some part $p_1^3$ merges with $p_2^3$). Assume that $r_{1,2}$ is the smallest of them, so we use $f(1,2,r)$ to describe the intersection. It follows that $r= \argmin_r d(f(1,2,r), s_{ext}^j(3))$. In Figure \ref{fig:intersection}, we can see that $q \in t_2$, so $r_2=r$. The maximum speed ratio is $r_{max}= \min \{ r',r_1,r_2 \}$.

\section{Conclusions}
\label{sec:conclusion}

In this paper, a variation of the art gallery problem was addressed. A team of mobile guards with finite speed tries to maintain visibility of a set of unpredictable intruders in a simply connected polygonal environment. The guards are deployed and confined to move along diagonals of the polygon. 
In this work, we presented a strategy to determine if $\lfloor \frac{n}{4} \rfloor$ are sufficient to track all the intruders forever. The problem was modeled as an hybrid automaton, which definition requires the concept of that determine an appropriate coordination to cover each triangle according to the location of the intruders. We also presented a method to determine when additional guards are required to ensure tracking. 

A few of our future research directions are as follows. One of the ongoing efforts is to consider the specific case of orthogonal polygons, since many indoor environments for practical cases can be modeled as orthogonal polygons, so we can take advantage of some of their properties to give better results for such cases. To achieve this, one potential direction is to use the quadrilateralization of these polygons instead of the triangulation and to define cut guards instead of diagonal guards, so we can reduce the number of guards required to track an intruder or a set of intruders.

\addtolength{\textheight}{-12cm}   









\bibliographystyle{IEEEtran}
\bibliography{IEEEabrv,references3}

\end{document}